\documentclass[letterpaper, 10 pt, conference]{ieeeconf}  

\IEEEoverridecommandlockouts                              

\usepackage{amsthm}
\usepackage{amsmath}
\usepackage{hyperref}
\usepackage{subcaption}
\usepackage{booktabs}
\usepackage{graphicx}
\usepackage{amsfonts}
\usepackage{multirow}
\theoremstyle{definition}
\newtheorem{theorem}{Theorem}
\newtheorem{lemma}[theorem]{Lemma} 
\usepackage{xcolor}
\usepackage{algorithm}  
\usepackage{algpseudocode}
\algrenewcommand\algorithmiccomment[1]{\hfill\textcolor{blue}{$\triangleright$ #1}}
\newcommand{\LineComment}[1]{\State \textcolor{blue}{$\triangleright$ \textit{#1}}}

\algtext*{EndFunction}

\newif\ifWithAppendix
\WithAppendixtrue   

\title{\LARGE \bf
Robust Differentiable Collision Detection for General Objects 
}

\author{Jiayi Chen$^{1, 2}$, Wei Zhao$^{3, 4}$, Liangwang Ruan$^{1, 2}$, Baoquan Chen$^{1}$, He Wang$^{1, 2\dagger}$
\thanks{$^{1}$Peking University. $^{2}$Galbot. $^{3}$Tsinghua University. $^{4}$The Hong Kong Polytechnic University. $^\dagger$Corresponding author: \href{mailto:hewang@pku.edu.cn}{hewang@pku.edu.cn}.}
}

\begin{document}

\maketitle


\begin{abstract}
Collision detection is a core component of robotics applications such as simulation, control, and planning. Traditional algorithms like GJK+EPA compute \textit{witness points}—the closest or deepest-penetration pairs between two objects—but are inherently non-differentiable, preventing gradient flow and limiting gradient-based optimization in contact-rich tasks such as grasping and manipulation. Recent work introduced efficient first-order randomized smoothing to make witness points differentiable; however, their direction-based formulation is restricted to convex objects and lacks robustness for complex geometries. In this work, we propose a robust and efficient differentiable collision detection framework that supports both convex and concave objects across diverse scales and configurations. Our method introduces distance-based first-order randomized smoothing, adaptive sampling, and equivalent gradient transport for robust and informative gradient computation. Experiments on complex meshes from DexGraspNet and Objaverse show significant improvements over existing baselines. Finally, we demonstrate a direct application of our method for dexterous grasp synthesis to refine the grasp quality. The code is available at \href{https://github.com/JYChen18/DiffCollision}{https://github.com/JYChen18/DiffCollision}.
\end{abstract}

\section{Introduction}

Collision detection is a long-standing and fundamental problem in robotics and computer graphics, with applications ranging from simulation and control to planning and learning. A key outcome of collision detection is the \emph{witness points}~\cite{ericson2004real}—the closest points when objects are separated or the deepest penetration points when they overlap. 
These points specify where and how objects interact: they directly indicate contact positions and allow computation of contact distance and direction, making them a cornerstone of contact modeling~\cite{SiggraphContact22, le2024contact}. 

Despite the critical role of these witness points, traditional algorithms 
such as GJK~\cite{gilbert2002fast, montaut2024gjk++}, EPA~\cite{van2001proximity}, and MPR~\cite{snethen2008xenocollide} are non-differentiable, preventing gradients from propagating through them. This lack of differentiability is one possible reason limiting the potential of gradient-based optimization in contact-rich tasks such as grasping and manipulation. As a result, robotics has often relied on gradient-free approaches such as reinforcement learning~\cite{rajeswaran2017learning, chen2023visual} and sampling-based planners~\cite{howell2022}, in contrast to the success of gradient-based methods in deep learning.

Recent work has begun improving the differentiability of contact modeling~\cite{li2020incremental, pang2023global, jin2024complementarity}, but mainly focuses on improving the contact solvers and smoothing the \emph{magnitude} of contact forces with respect to inter-object distance, known as \emph{force-at-a-distance}. In contrast, the differentiability of witness points themselves has received little attention. This omission is critical: without gradients through witness points, it is difficult to enforce desired contact positions or adjust force directions on complex shapes such as concave regions (Figure~\ref{fig: motivation}).

To bridge this gap, Montaut et al.~\cite{montaut2022differentiable} introduced efficient first-order randomized smoothing approaches to make witness point differentiable. However, their direction-based method is limited to convex objects and has only been extensively tested on simple meshes (as few as 12 vertices), falling short of the demands of real robotic applications.

\begin{figure}
    \centering
    \includegraphics[width=1.0\linewidth]{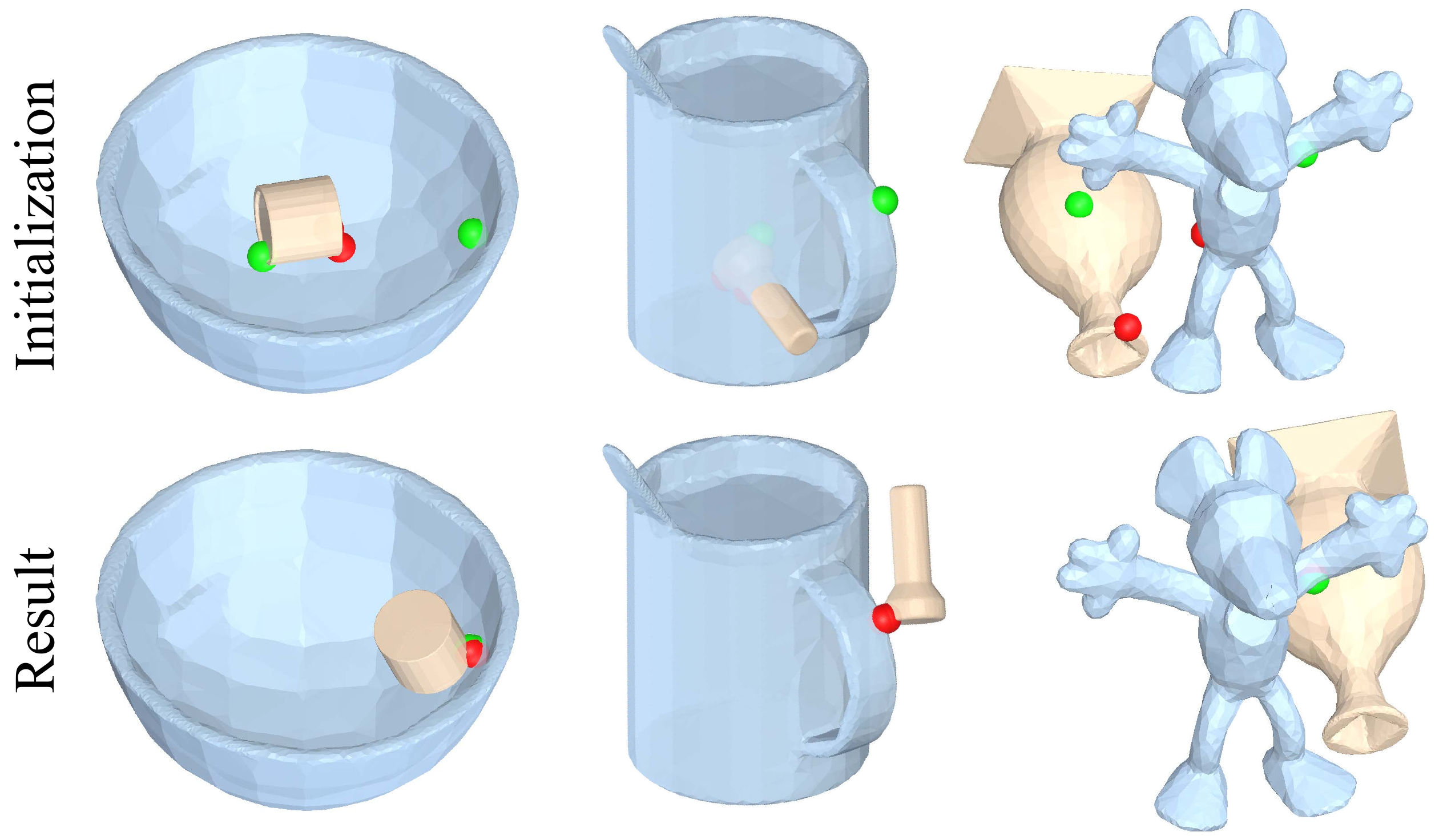}
    \caption{\textbf{Task illustration.} We solve for the relative pose that aligns target points (green) with witness points (red). This task validates the computed derivatives of the witness points \textit{with respect to} the object poses.}
    \label{fig: teaser}
\end{figure}

\begin{figure}
    \centering
    \includegraphics[width=1.0\linewidth]{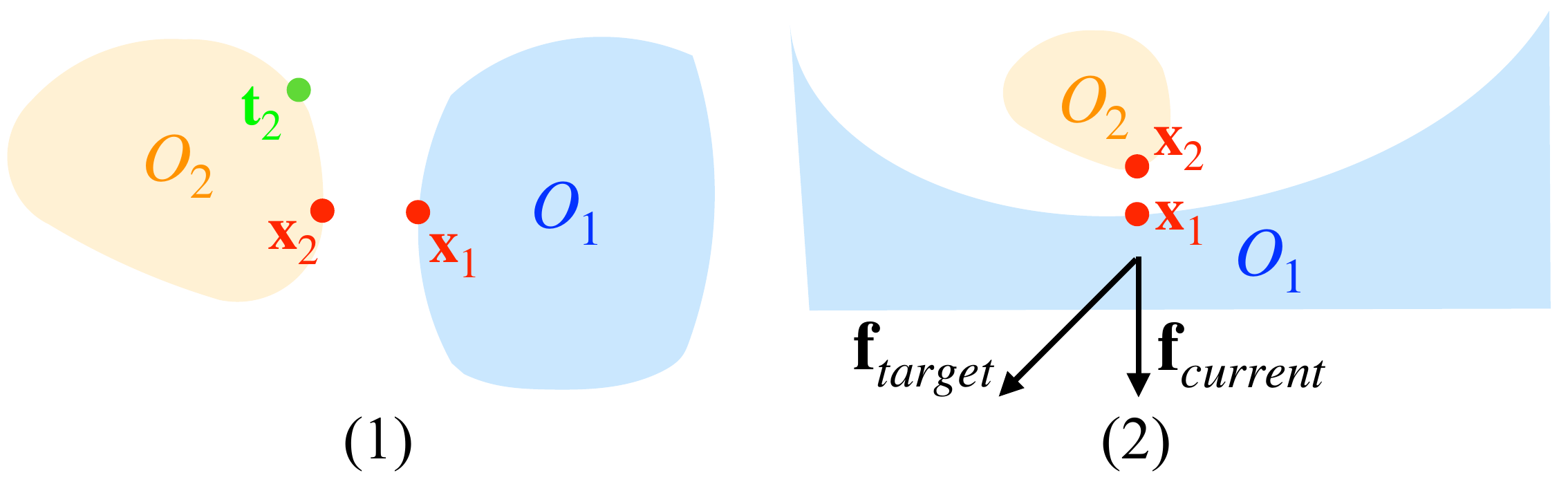}
    \caption{\textbf{Motivation.} When optimizing the pose $T_2$ of object $O_2$, treating the witness points $\mathbf{x}_1, \mathbf{x}_2$ as fixed on objects $O_1$ and $O_2$ without considering their derivatives cannot handle \textit{at least} two cases: (1) enforcing a specified point $\mathbf{t}_2$ to be a witness point (i.e., as the task in Fig.~\ref{fig: teaser}); (2) enforcing object $O_2$ to produce a target contact force $\mathbf{f}_{\text{target}}$ on a concave object $O_1$, where ignoring $\partial \mathbf{x}_1 / \partial T_2$ actually assumes that $\mathbf{x}_1$ is fixed. In this situation, to generate a leftward contact force, one would move $\mathbf{x}_2$ rightward, which is incorrect. By contrast, our method accounts for $\partial \mathbf{x}_1 / \partial T_2$, recognizing that when $\mathbf{x}_2$ moves rightward, $\mathbf{x}_1$ also shifts rightward even more. Thus, the correct update is to move $O_2$ leftward.}
    \label{fig: motivation}
\end{figure}

In this work, we propose a general approach to calculate the derivative of witness points that applies to both convex and concave objects. As shown in Figure~\ref{fig: teaser}, our method is robust to object scale, handles both penetrating and non-penetrating configurations, and scales to complex shapes such as those in DexGraspNet~\cite{wang2022dexgraspnet} and Objaverse~\cite{deitke2023objaverse}.

Our method belongs to the family of first-order randomized smoothing but introduces three key innovations. First, instead of relying on direction–based smoothing derived from GJK optimality conditions~\cite{montaut2022differentiable}, we propose \textit{distance–based smoothing}, which generalizes naturally beyond convex shapes. Second, rather than approximating local geometry from neighboring vertices—which is mesh-sensitive and not parallel-friendly—we propose an \textit{adaptive sampling} strategy,  yielding greater robustness. Third, we introduce \emph{equivalent gradient transport}, which improves gradient quality when only one object’s pose is optimized, a scenario frequently encountered in grasping and manipulation.

Experiments on both convex and concave objects demonstrate significant improvements over existing baselines. On complex meshes from DexGraspNet and Objaverse, our method achieves a median error below 0.1 mm across 400 random object pairs with 1024 tasks each. At mm-level accuracy, we outperform baselines by more than $40\%$. Our approach is also memory- and time-efficient, and easily parallelizable on GPUs. Finally, we demonstrate a direct application to dexterous grasp refinement, achieving improved grasp quality.

In summary, our contributions are:
\begin{itemize}
    \item An open-source, robust, and efficient differentiable collision detection framework for both convex and concave objects, built upon distance-based smoothing with adaptive sampling and equivalent gradient transport;
    \item Competitive empirical performance on large-scale object datasets, demonstrating the robustness of our method to object shapes, scales and configurations;
    \item A preliminary application to dexterous grasp refinement.
\end{itemize}

\section{Related Work}

\subsection{Discrete Collision Detection for Meshes}

A major class of collision detection algorithms leverages the \textit{Minkowski sum} to reduce collision detection to an origin-in-polytope test. These methods are highly efficient and widely used in rigid-body simulators such as MuJoCo~\cite{todorov2012mujoco} and PyBullet~\cite{coumans2016pybullet}. Representative algorithms include the Gilbert-Johnson-Keerthi (GJK) algorithm~\cite{gilbert2002fast, montaut2024gjk++}, which iteratively searches for the closest point to the origin in the Minkowski difference; the Expanding Polytope Algorithm (EPA)~\cite{van2001proximity}, which refines the GJK result to compute penetration depth and witness points; and the Minkowski Portal Refinement (MPR) algorithm~\cite{snethen2008xenocollide}, which combines elements of both for practical efficiency. These methods assume convex meshes, and concave objects typically require convex decomposition~\cite{mamou2009simple, wei2022coacd}.

Another line of work relies on exact triangle intersection tests, which check for intersections between mesh triangles. While computationally more expensive, these methods are naturally parallelizable on GPUs. A common variant is the Separating Axis Test (SAT)~\cite{gottschalk1996separating}, again limited to convex shapes. More general methods explicitly check vertex–face and edge–edge pairs, allowing them to handle concave or deformable objects. However, they cannot identify deepest penetration points and typically require penetration-free configurations at each timestep~\cite{li2020incremental}, making them less common in robotics applications.

All of the above methods are non-differentiable, except for vertex–face and edge–edge tests, which provide highly limited and local gradients. In contrast, our method computes derivatives of witness points that are compatible with any forward collision detection algorithm, enabling gradient flow through witness points.

\subsection{Differentiable Collision Detection} 

Interest in differentiable collision detection has grown only recently, largely motivated by the progress of differentiable simulation~\cite{de2018end, freeman2021brax,warp2022}.  
Tracy et al.~\cite{tracy2022differentiable} formulate collision detection as a convex optimization problem of scaling the two objects, which naturally yields derivatives. However, their method is restricted to convex primitives and scales poorly to complex meshes compared to GJK.
Montaut et al.~\cite{montaut2022differentiable} introduce random smoothing techniques to approximate derivatives. The zeroth-order method, akin to finite differences, requires multiple collision queries per backward pass and is thus inefficient. The first-order method leverages the optimality conditions of the GJK algorithm, offering efficiency but restricted to strictly convex shapes and not robust enough for complex meshes. 

Inspired by~\cite{montaut2022differentiable}'s first-order method, we propose a general approach for both convex and concave objects with improved robustness, while preserving computational efficiency.  

\section{Preliminaries on the $\mathrm{SE}(3)$ Lie Group}


\paragraph{Lie algebra} 
The tangent space of $\mathrm{SE}(3)$ at the identity element is denoted as Lie algebra $\mathfrak{se}(3)$:
\begin{equation}
\mathfrak{se}(3) = \left\{
\begin{bmatrix}
[\omega]_\times & v \\
0 & 0
\end{bmatrix} \;\middle|\;
\omega, v \in \mathbb{R}^3 \right\}
\end{equation}
where $[\omega]_\times$ is the skew-symmetric matrix of $\omega$. The standard exponential map
\begin{equation}
    \exp: \mathfrak{se}(3)\to \mathrm{SE}(3)
\end{equation}
is a local diffeomorphism around the identity and serves as a natural
retraction, mapping elements of the tangent space back onto the group
$\mathrm{SE}(3)$.

\paragraph{Adjoint operator}  
For $T \in \mathrm{SE}(3)$, the adjoint operator describes how a rigid-body
transformation conjugates Lie algebra elements:
\begin{equation}
    \mathrm{Ad}_T : \mathfrak{se}(3) \to \mathfrak{se}(3), \quad
    \mathrm{Ad}_T(\xi) \;=\; T \, \xi \, T^{-1}.
\end{equation}
The adjoint operator satisfies the fundamental identity
\begin{equation}\label{eq: adjoint}
    T \exp(\xi) T^{-1} = \exp\!\big( (\mathrm{Ad}_T \xi) \big),
\end{equation}
which is widely used in robotics to transport velocities, Jacobians, and
differentials between coordinate frames.

For further details on $\mathrm{SE}(3)$, its Lie algebra, and the adjoint operator, we refer the reader to \cite{sola2018micro}.

\section{Methods}

\begin{figure}
    \vspace{2mm}
    \centering
    \includegraphics[width=1.0\linewidth]{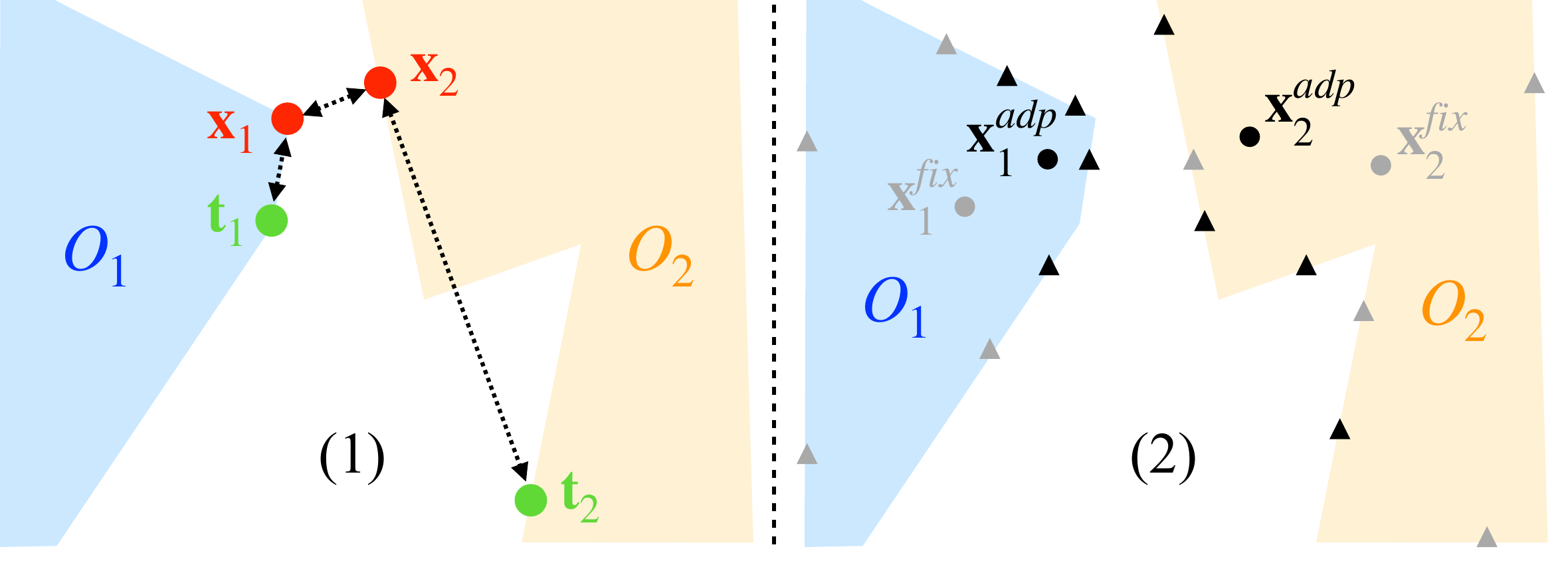}
    \caption{\textbf{(1) Task formulation.} The witness points $\mathbf{x}_1,\mathbf{x}_2$ calculated by collision detection are expected to match specified target points $\mathbf{t}_1,\mathbf{t}_2$ via the losses shown by dotted lines.
   \textbf{(2) Smoothing witness points.} Adaptive sampling yields better surface samples (triangles) than fixed sampling, improving the approximation of $\mathbf{x}_1,\mathbf{x}_2$.}
   \vspace{-3mm}
    \label{fig: method}
\end{figure}

\subsection{Task Formulation}

Let $O_i$ be objects with poses $T_i \in SE(3)$ for $i \in \{1,2\}$. We denote world-frame points as $\mathbf{p}_i = T_i \mathbf{p}_{i,o}$, where $\mathbf{p}_{i,o}$ is the representation in the local frame of $O_i$. Next, we define:
\begin{itemize}
    \item \textbf{Witness points} $\mathbf{x}_i$: The closest/penetrating points between $O_1$ and $O_2$ computed via GJK/EPA. $i \in \{1,2\}$.
    \item \textbf{Target points} $\mathbf{t}_i$: The target contact points specified per task and fixed in the local frame of $O_i$. $i \in \{1,2\}$.
    \item \textbf{Boundary samples} $\{\mathbf{v}_{i,j}\}_{j=1}^N$: $N$ points randomly sampled to approximate the object surface geometry.
\end{itemize}

The objective is to optimize the poses $T_1, T_2$ such that the witness points approach the targets:
\begin{equation}
\label{eq: loss}
    \min_{T_1, T_2} \mathcal{L} = \|\mathbf{x}_1 - \mathbf{x}_2\|^2 
      + \|\mathbf{x}_1 - \mathbf{t}_1\|^2 
      + \|\mathbf{x}_2 - \mathbf{t}_2\|^2
\end{equation}
where $\mathbf{x}_1$ and $\mathbf{x}_2$ depend on $T_1$ and $T_2$. Our proposed algorithm calculates $\nabla_{T_i}\mathcal{L}$ through $\mathbf{x}_i$, as shown in Algorithm~\ref{alg:diff_collision}.

\subsection{Distance–Based Softmax Smoothing}
\label{sec:smooth}
Without loss of generality, we focus on $\mathbf{x}_1$, as $\mathbf{x}_1$ and $\mathbf{x}_2$ are symmetric and the results for $\mathbf{x}_2$ follow analogously.

We begin by reformulating $\mathbf{x}_1$ as
\begin{equation}
    \mathbf{x}_1 = \underset{\mathbf{v}\in T_1 \partial O_1}{\operatorname{argmax}} \; \big(-\|\mathbf{v}-\mathbf{x}_2\|^2 \big),
\end{equation}
where $\partial O_1$ denotes the boundary of object $O_1$ and $T_1 \partial O_1$ denotes the object boundary in the world frame. This formulation is exact when the two objects do not overlap. In penetration cases, it becomes an approximation, but we find empirically that our method still performs well—likely because the penetration is penalized by the first term in the loss function (Eq.~\ref{eq: loss}). Since penetration is typically undesirable in robotic tasks, such a term is commonly included in gradient-based methods, making this formulation a reasonable approximation.

To obtain a differentiable surrogate, we sample $N$ points $\{\mathbf{v}_{1,j}\}_{j=1}^N\in T_1\partial O_1$ and replace the non-smooth $\operatorname{argmax}$ with a softmax relaxation. Specifically, we define the score value $u_j\in\mathbb{R}$ and softmax weights $w_j\in\mathbb{R}$
\begin{equation}
    u_j = -\|\mathbf{v}_{1,j} - \mathbf{x}_2\|^2, 
    \quad w_j = \frac{\exp (u_j / \tau)}
    {\sum_{k=1}^N \exp (u_k / \tau)}
\end{equation}
where $\tau$ is the temperature parameter. In practice, we find that setting $\tau=\operatorname{std}(\mathbf{u})$ (i.e., the standard deviation of $u_i$) works well, as it adapts automatically to object scale and avoids manual tuning.
The witness point is then approximated by the weighted combination of sampled points:
\begin{equation}\label{eq:smoothing}
\mathbf{x}_1 \approx \mathbf{x}_1^\star = \sum_{j=1}^N w_j \mathbf{v}_{1, j} = \mathbf{V}_1\mathbf{w}_1^T,
\end{equation}
where $\mathbf{w}_1=[w_1, ...,w_N]\in \mathbb{R}^N$ and $\mathbf{V}_1=[\mathbf{v}_{1,1},...,\mathbf{v}_{1,N}]\in\mathbb{R}^{3\times N}$. If $\mathbf{x}_{1,o} \in \{ \mathbf{v}_{1,j,o} \}$, then $\|\mathbf{x}_1 - \mathbf{x}_1^\star \|\to 0$ as $\tau \to 0$. 

\begin{algorithm}[t]
\caption{Robust Differentiable Collision Detection}
\label{alg:diff_collision}
\begin{algorithmic}[1]
\Require Objects $O_1, O_2$;Poses $T_1, T_2$;Gradients $\nabla_{\mathbf{x}_1}\mathcal{L}, \nabla_{\mathbf{x}_2}\mathcal{L}$
\Ensure Witness points $\mathbf{x}_1, \mathbf{x}_2$; Gradients $\nabla_{T_1}\mathcal{L}, \nabla_{T_2}\mathcal{L}$

\Function{Forward}{$T_1, T_2$}
    \LineComment{\textit{Compute witness points using standard solver}}
    \State $\mathbf{x}_1, \mathbf{x}_2 \gets \text{GJK}(O_1, O_2, T_1, T_2)$
    \State \Call{SaveForBackward}{$O_1, O_2, T_1, T_2, \mathbf{x}_1, \mathbf{x}_2$}
    \State \Return $\mathbf{x}_1, \mathbf{x}_2$
\EndFunction

\Statex

\Function{Backward}{$\nabla_{\mathbf{x}_1}\mathcal{L}, \nabla_{\mathbf{x}_2}\mathcal{L}$}
    \LineComment{\textit{Fixed or Adaptive Sampling (Sec.~\ref{sec:sampling})}}
    \State $\mathbf{V}_1, \mathbf{V}_2 \gets \text{SampleSurface}(O_1, O_2, T_1, T_2, \mathbf{x}_1, \mathbf{x}_2)$
    
    \LineComment{\textit{Distance-based Randomized Smoothing} (Sec.~\ref{sec:smooth})}
    \State $\mathbf{x}_1^\star, \mathbf{x}_2^\star \gets \text{WeightedSum}(\mathbf{V}_1, \mathbf{V}_2, \mathbf{x}_1, \mathbf{x}_2)$ \Comment{Eq.~\ref{eq:smoothing}}
    
    \LineComment{\textit{Derivative Calculation (Sec.~\ref{sec:derivative})}}
    \State $\frac{\partial \mathbf{x}_1}{\partial T_1}, \frac{\partial \mathbf{x}_2}{\partial T_2} \gets \text{ComputeSelfJacob}(\mathbf{x}_1^\star, \mathbf{x}_2^\star)$  \Comment{Eq.~\ref{eq:approximate}}
    \State $\frac{\partial \mathbf{x}_1}{\partial T_2}, \frac{\partial \mathbf{x}_2}{\partial T_1} \gets \text{ComputeCrossJacob}(\mathbf{x}_1^\star, \mathbf{x}_2^\star)$  \Comment{Eq.~\ref{eq:crossjacobi}}
    
    \State $\nabla_{T_1}\mathcal{L} \gets \nabla_{\mathbf{x}_1}\mathcal{L} \cdot \frac{\partial \mathbf{x}_1}{\partial T_1} + \nabla_{\mathbf{x}_2}\mathcal{L} \cdot \frac{\partial \mathbf{x}_2}{\partial T_1}$ \Comment{Chain rule}
    \State $\nabla_{T_2}\mathcal{L} \gets \nabla_{\mathbf{x}_1}\mathcal{L} \cdot \frac{\partial \mathbf{x}_1}{\partial T_2} + \nabla_{\mathbf{x}_2}\mathcal{L} \cdot \frac{\partial \mathbf{x}_2}{\partial T_2}$ 
    
    \LineComment{\textit{Equivalent Gradient Transport (Sec.~\ref{sec:egt})}}
    \State $\nabla_{T_1}\mathcal{L}, \nabla_{T_2}\mathcal{L} \gets \text{EGT}(T_1, T_2, \nabla_{T_1}\mathcal{L}, \nabla_{T_2}\mathcal{L})$ 
    
    \State \Return $\nabla_{T_1}\mathcal{L}, \nabla_{T_2}\mathcal{L}$
\EndFunction
\end{algorithmic}
\end{algorithm}

\subsection{Sampling Strategies}\label{sec:sampling}
The sampled points $\mathbf{v}_{1,j,o}$ provide a discrete approximation of the object surface and are thus critical to the accuracy of the witness point approximation. To balance quality and efficiency, we first generate a large candidate set (hundreds of points) as a preprocessing step, and then select a small subset (e.g., 16 points) in each iteration as the local approximation.

Formally, let $P_{\mathrm{static}} = P_{\mathrm{ver}} \cup P_{\mathrm{sur}}$ be the fixed set of object vertices and random surface samples. At each iteration, we augment this with the current witness point $\mathbf{x}_{1,o}$ to form the full candidate set $P = P_{\mathrm{static}} \cup \{\mathbf{x}_{1,o}\}$. 
We then select the active local subset by filtering $P$ based on the Euclidean distance to $\mathbf{x}_{1,o}$:
\begin{equation}
    \{\mathbf{v}_{1,j,o}\} = \{\mathbf{v}_{1,j,o}\in P : \|\mathbf{v}_{1,j,o} - \mathbf{x}_{1,o} \| \leq \alpha \}.
\end{equation}

We consider two strategies for choosing $\alpha$:
\begin{itemize}
    \item \textbf{Fixed Sampling}: $\alpha$ is a constant hyperparameter used throughout the optimization.
    \item \textbf{Adaptive Sampling}: $\alpha$ is dynamically updated at each iteration as $\alpha = \max\big(\|\mathbf{t}_1 - \mathbf{x}_1\|, \varepsilon\big)$,
where $\mathbf{t}_1$ is the target contact point and $\varepsilon$ is a small positive constant for numerical stability. Adaptive sampling is particularly helpful for concave objects with complex geometries.
\end{itemize}

\subsection{Derivative Calculation}
\label{sec:derivative}
Applying the chain rule on Eq.~\ref{eq:smoothing}:
\begin{equation}
    \frac{\partial \mathbf{x}_1}{\partial T_1} \approx \frac{\partial \mathbf{x}_1^\star}{\partial T_1}
    = \frac{\partial \mathbf{x}_1^\star}{\partial \mathbf{V}_1} \frac{\partial \mathbf{V}_1}{\partial T_1} 
    + \frac{\partial \mathbf{x}_1^\star}{\partial \mathbf{w}_1} \frac{\partial \mathbf{w}_1}{\partial T_1}.
\end{equation}
When differentiating with respect to $T_1$, we treat $\mathbf{x}_2$ as fixed. Under this assumption, the expression becomes
\begin{equation}\label{eq:approximate}
    \frac{\partial \mathbf{x}_1}{\partial T_1} \approx 
    \frac{\partial \mathbf{x}_1^\star}{\partial \mathbf{V}_1} \frac{\partial \mathbf{V}_1}{\partial T_1} 
    + \frac{\partial \mathbf{x}_1^\star}{\partial \mathbf{w}_1} 
      \frac{\partial \mathbf{w}_1}{\partial \mathbf{V}_1} 
      \frac{\partial \mathbf{V}_1}{\partial T_1}.
\end{equation}
This approximation is exact when $\mathbf{x}_2$ is a vertex of $T_2 O_2$, since small perturbations of $T_1$ leave its position unchanged. More generally, we find empirically that the same assumption also yields a reasonable approximation when $\mathbf{x}_2$ lies on the surface of $O_2$.  
An analogous derivation gives $\partial \mathbf{x}_2 / \partial T_2$.  

Next, we consider the cross derivative of $\mathbf{x}_1$ with respect to $T_2$. Since $\partial \mathbf{V}_1/\partial T_2 = 0$, we obtain
\begin{equation}\label{eq:crossjacobi}
    \frac{\partial \mathbf{x}_1}{\partial T_2} \approx \frac{\partial \mathbf{x}_1^\star}{\partial T_2} 
    = \frac{\partial \mathbf{x}_1^\star}{\partial \mathbf{w}_1}
      \frac{\partial \mathbf{w}_1}{\partial \mathbf{x}_2}
      \frac{\partial \mathbf{x}_2}{\partial T_2}.
\end{equation}
Because an approximation of $\partial \mathbf{x}_2 / \partial T_2$ has already been derived as Eq.~\ref{eq:approximate}, it can be directly substituted into the above expression to evaluate $\partial \mathbf{x}_1 / \partial T_2$, thereby avoiding reusing the assumption that $\mathbf{x}_{2,o}$ is fixed. By symmetry, the cross derivative $\partial \mathbf{x}_2 / \partial T_1$ can be obtained in the same manner.

\subsection{Equivalent Gradient Transport}
\label{sec:egt}
In practical robotic tasks such as grasping, it is common to optimize only the robot’s pose while keeping the object fixed. Suppose $T_1$ is fixed during optimization. In this case, directly updating $T_2$ alone can be inefficient, because the task’s inherent symmetry is broken and $T_2$ may need larger adjustments than $T_1$ to reach the target points $\mathbf{t}_1$ if the initialization is far from the target. To tackle it, we propose a strategy to transfer the update that would nominally apply to $T_1$ to $T_2$ in an equivalent manner, as shown in Figure~\ref{fig: equiv grad}.

Given the gradient in the ambient space $G_i = \partial L / \partial T_i \in \mathbb{R}^{4\times 4}$,
the corresponding Lie algebra element is obtained as
\begin{equation}
    \xi_i = \mathrm{Proj}_{\mathfrak{se}(3)}\!\big( T_i^{-1} G_i \big), 
    \quad i = 1,2,
\label{eq: proj grad}
\end{equation}
where $\mathrm{Proj}_{\mathfrak{se}(3)}(\cdot)$ denotes the projection of a $4\times 4$ matrix onto the linear space $\mathfrak{se}(3)$. The following lemma states how to obtain the \emph{equivalent gradient} transported from $T_1$ to $T_2$.

\begin{lemma}\label{lem:grad_tran} 
For any $T_1,T_2\in \mathrm{SE}(3)$, $\xi_1\in\mathfrak{se}(3)$, and $\lambda\in\mathbb{R}$, updating $T_1$ with $-\lambda\xi_1$ is equivalent, in terms of the relative pose, to updating $T_2$ with $-\lambda\tilde{\xi}_2$, where 
\begin{equation}
    \tilde{\xi}_2 = - \mathrm{Ad}_{T_2^{-1}T_1}\xi_1.
\end{equation}
Formally, the equivalent relative pose is given by
\begin{equation}\label{eq:transport-identity}
    \left(T_1 \exp(-\lambda \xi_1)\right)^{-1}T_2
    = T_1^{-1}\!\left(T_2 \exp(-\lambda \tilde{\xi}_2)\right).
\end{equation}
\end{lemma}
\begin{proof}
Since $(\exp(-\lambda \xi_1))^{-1} = \exp(\lambda \xi_1)$, we have
\begin{equation}
\big(T_1 \exp(-\lambda \xi_1)\big)^{-1}T_2
= \exp(\lambda \xi_1)T_1^{-1}T_2.
\end{equation}
By the adjoint identity Eq.~\ref{eq: adjoint},
\begin{equation}
\exp(\lambda \xi_1)T_1^{-1}T_2
= T_1^{-1}T_2 \exp\!\big(\lambda\,\mathrm{Ad}_{T_2^{-1}T_1}\xi_1\big).
\end{equation}
Substituting $\tilde{\xi}_2=-\mathrm{Ad}_{T_2^{-1}T_1}\xi_1$ yields Eq.~\ref{eq:transport-identity}.
\end{proof}

\begin{figure}
    \vspace{2mm}
    \centering
    \includegraphics[width=\linewidth]{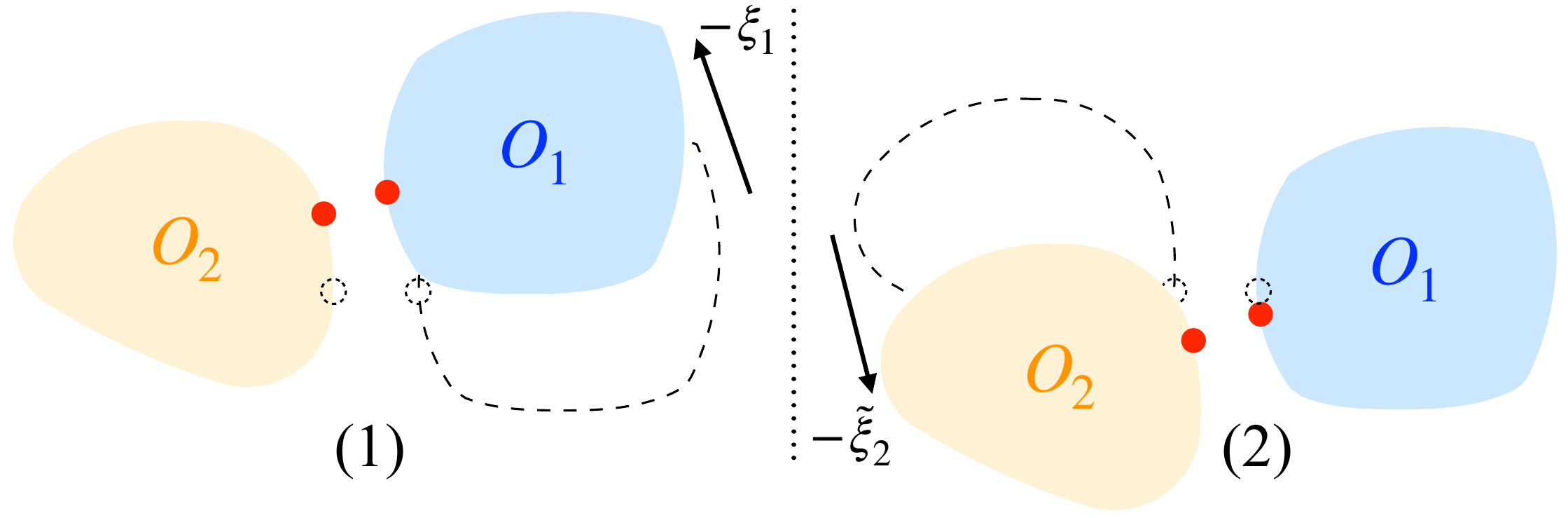}
    \caption{\textbf{Equivalent gradient transport (EG).} Updating $T_1$ with gradient $\xi_1$ produces the same relative pose as updating $T_2$ with our proposed \textit{equivalent gradient} $\tilde{\xi}_2$. The object and witness points before the update are shown as dotted lines.}
    \label{fig: equiv grad}
\end{figure}

It is worth noting that the equivalent gradient transport technique is not limited to cases where $T_1$ and $T_2$ are decision variables. It can be implemented as a plug-in module and applied when $T_1$ and $T_2$ act as intermediate transformations within a larger computation graph, as long as we project the obtained gradient $\xi_2+\tilde{\xi}_2$ from $\mathfrak{se}(3)$ back to the ambient space using the inverse of Eq.~\ref{eq: proj grad}. For example, in the application shown in Section~\ref{sec: applic}, $T_2$ is calculated by forward kinematics and this module can still be used.

\begin{table*}[]
    \vspace{2mm}
    \centering
    \begin{tabular}{l|ccc|ccc|ccc|ccc}
    \hline
    \multirow{2}{*}{Method} & \multicolumn{3}{c|}{DexGraspNet (Convex)} & \multicolumn{3}{c|}{DexGraspNet (Concave)} & \multicolumn{3}{c|}{Objaverse (Convex)} & \multicolumn{3}{c}{Objaverse (Concave)} \\
    & D5$\downarrow$ & D9$\downarrow$ & Acc(\%)$\uparrow$  & D5$\downarrow$ & D9$\downarrow$ & Acc(\%)$\uparrow$ & D5$\downarrow$ & D9$\downarrow$ & Acc(\%)$\uparrow$ & D5$\downarrow$ & D9$\downarrow$ & Acc(\%)$\uparrow$ \\
    \hline
    Analytical & 1.0e-4 & 2.0e-3 & 4.0 & 3.3e-5 & 1.2e-3 & 15.2 & 5.1e-5 & 1.4e-3 & 8.9 & 3.3e-5 & 8.9e-4 & 13.7 \\
    Finite Difference & 2.7e-6 & 1.5e-4 & 40.6 & 3.9e-6 & 1.8e-4 & 35.9 & 2.5e-6 & 1.7e-4 & 41.3 & 8.0e-6 & 2.6e-4 & 27.9 \\
    RS-0~\cite{montaut2022differentiable} & 1.5e-6 & 1.3e-4 & 45.3 & 2.6e-6 & 1.2e-4 & 38.0 & 1.5e-6 & 1.8e-4 & 46.0 & 8.6e-6 & 2.9e-4 & 27.4 \\
    RS-1-Dir~\cite{montaut2022differentiable} & 2.2e-6 & 4.0e-4 & 45.6 & 1.7e-5 & 1.7e-3 & 28.3 & 3.8e-7 & 3.2e-4 & 55.6 & 2.8e-5 & 1.8e-3 & 22.1 \\
    \hline
    Ours & \textbf{4.3e-9} & \textbf{5.6e-7} & \textbf{91.1} & \textbf{6.5e-9} & \textbf{9.0e-6} & \textbf{80.3} & \textbf{5.4e-9} & \textbf{1.2e-6} & \textbf{89.4} & \textbf{1.0e-7} & \textbf{4.5e-5} & \textbf{61.8} \\
    \hline
    \end{tabular}
    \caption{\textbf{Quantitative comparison.} Our method significantly outperforms all baselines across datasets and geometries.}
    \label{tab: main}
\end{table*}

\subsection{Comparison with First-order Method in \cite{montaut2022differentiable}}
\label{sec: compare}
Our method is inspired by~\cite{montaut2022differentiable} and therefore shares some components, such as softmax smoothing and point sampling. However, both are redesigned in our work for robustness.

For softmax smoothing, their score value is defined as 
\begin{equation}
    u_j = \left\langle \mathbf{v}_{1,j}, \mathbf{y} \right\rangle,~~~~\mathbf{y} = \begin{cases}
    \mathbf{x}_2 - \mathbf{x}_1, & \text{if no penetration} \\
    \mathbf{x}_1 - \mathbf{x}_2, & \text{otherwise.}
    \end{cases}
\end{equation}
This \emph{direction-based smoothing} leverages the optimality conditions of the GJK/EPA algorithm but comes with several limitations. Most importantly, it only works reliably when both objects are strictly convex. Using their score vector, points $\mathbf{v}_{1,j}$ lying near a plane orthogonal to $\mathbf{y}$ become indistinguishable, and concave regions cannot be handled at all because GJK algorithm is not applicable. In addition, numerical issues arise when $\mathbf{x}_1$ is extremely close to $\mathbf{x}_2$, making $\mathbf{y}$ nearly zero and the direction ill-defined
\ifWithAppendix
(see Appendix~\ref{apd: contact margin})
\fi 
. In contrast, our \emph{distance-based smoothing} extends naturally to general objects and avoids the numerical instability caused by vanishing directions.

For the sampling strategy, \cite{montaut2022differentiable} considers vertices within a fixed-level neighborhood (e.g., 3- or 5-ring) around the witness point, which is sensitive to mesh quality. In contrast, we address this issue by filtering candidates pre-sampled on object surfaces by their distance to the witness point.

Overall, our scheme (\textbf{Distance + Adaptive + EG}) substantially enhances robustness on general objects compared to the previous method~\cite{montaut2022differentiable} (\textbf{Direction + Neighbor}).
\begin{figure*}
    \centering
    \includegraphics[width=0.95\linewidth]{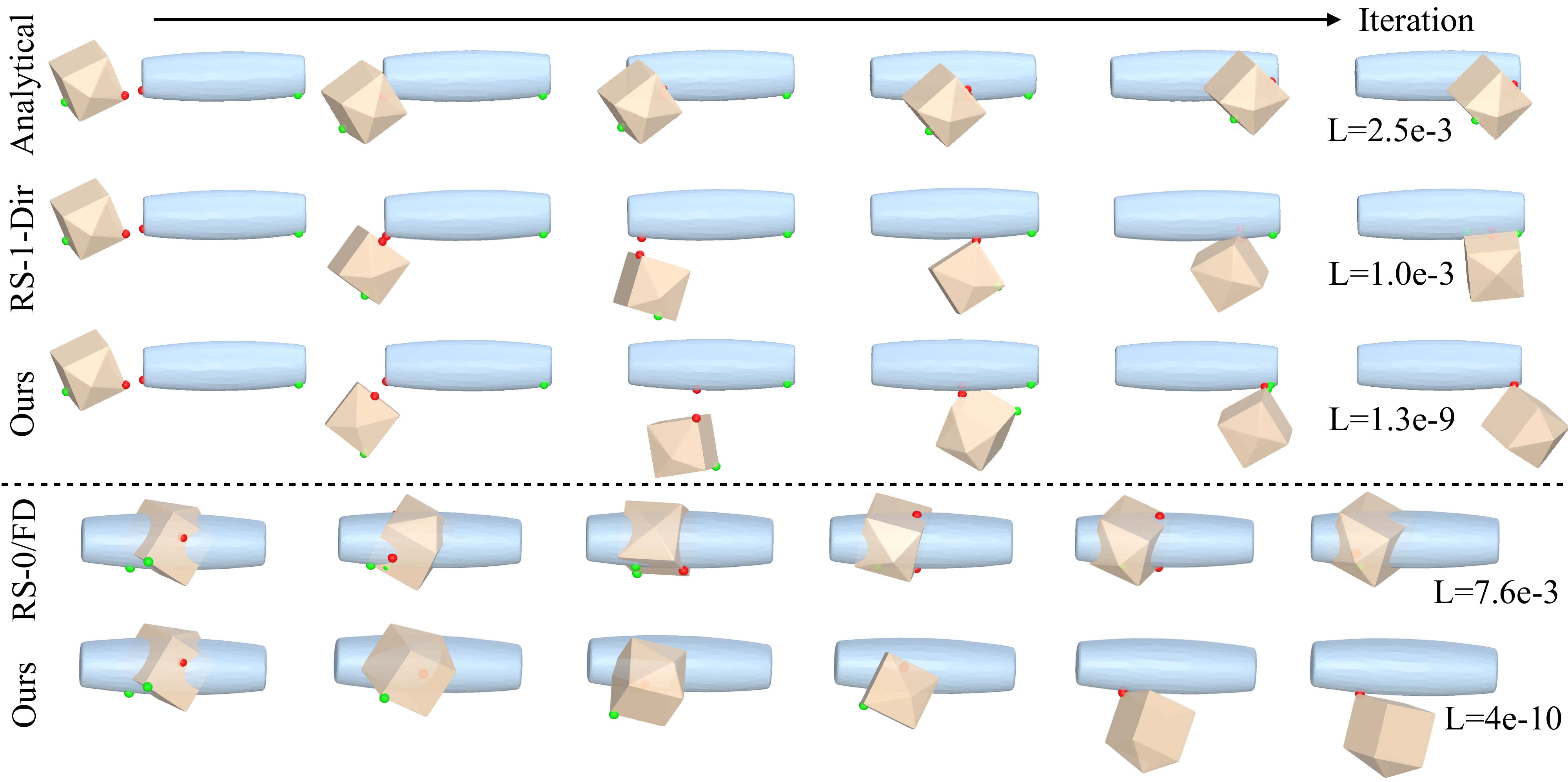}
    \caption{\textbf{Qualitative comparison on convex objects.} Different baselines exhibit distinct failure patterns: \textit{Analytical} often gets stuck due to zero derivatives at vertices (cols. 2, 5); \textit{RS-1-Dir} struggles to disambiguate vertices lying near a plane (col. 6); \textit{RS-0} and \textit{Finite Difference (FD)} often fail to resolve initial penetrations. In contrast, \textit{Ours} performs well in both scenarios.}
    \label{fig: convex exp}
\end{figure*}

\section{Experiments}

\subsection{Setup and Evaluation Metrics}

Our experimental setup differs from~\cite{montaut2022differentiable} in two key aspects. First, we evaluate on objects that are closer to real robotic applications, rather than rough convex shapes. Second, we use plain gradient descent without line search, to better highlight the quality of derivatives.

\textbf{Objects.} We randomly select assets from the 10k-object dataset used in Dexonomy~\cite{chen2025dexonomy}, where 5k are from DexGraspNet~\cite{wang2022dexgraspnet} and 5k from Objaverse~\cite{deitke2023objaverse}. Each object is scaled so its bounding box diagonal lies in $[0.01,0.2]$m. 

\textbf{Experimental Setup.} We evaluate two settings: (1) using the convex hull of each object (denoted as \textbf{convex}), and (2) applying CoACD~\cite{wei2022coacd} for convex decomposition (denoted as \textbf{concave}). For each setting, we fix 100 random object pairs shared across all methods. From each pair, 1,024 target point pairs $(\mathbf{t}_{1,o}, \mathbf{t}_{2,o})$ are sampled from mesh vertices or surfaces, resulting in 100k tasks per method. In each task, Object 1 is fixed while Object 2 is randomly initialized around it.

\textbf{Optimizer.} We run gradient descent for 2k iterations. To avoid numerical instability, pose derivatives are normalized, since they can be extremely large at the beginning and tiny near convergence. The rotation step size $s_r$ decays from $10 \rightarrow 1 \rightarrow 0.1$ at 200 and 1800 iterations, while the translation step size $s_t=s_r/100$.

\textbf{Implementation.} We use Coal~\cite{Pan_Coal_-_An_2025} for narrow-phase collision detection, parallelized with OpenMP~\cite{dagum1998openmp}. For concave objects, we add a broad-phase check using bounding spheres. Our proposed derivatives for witness points are implemented in PyTorch~\cite{paszke2019pytorch}, supporting batched processing of multiple target points on a GPU.

\textbf{Metrics.} (1) \textbf{D5}: median loss at the final iteration,  
(2) \textbf{D9}: ninth decile of final loss,  
(3) \textbf{Acc (\%)}: fraction of tasks with final loss $<$1e-6.  
For reference, since the loss sums three squared distances, a per-point displacement of 1 mm (1e-6 after squaring), yields a total loss of about 3e-6.

\begin{figure*}
    \centering
    \includegraphics[width=0.9\linewidth]{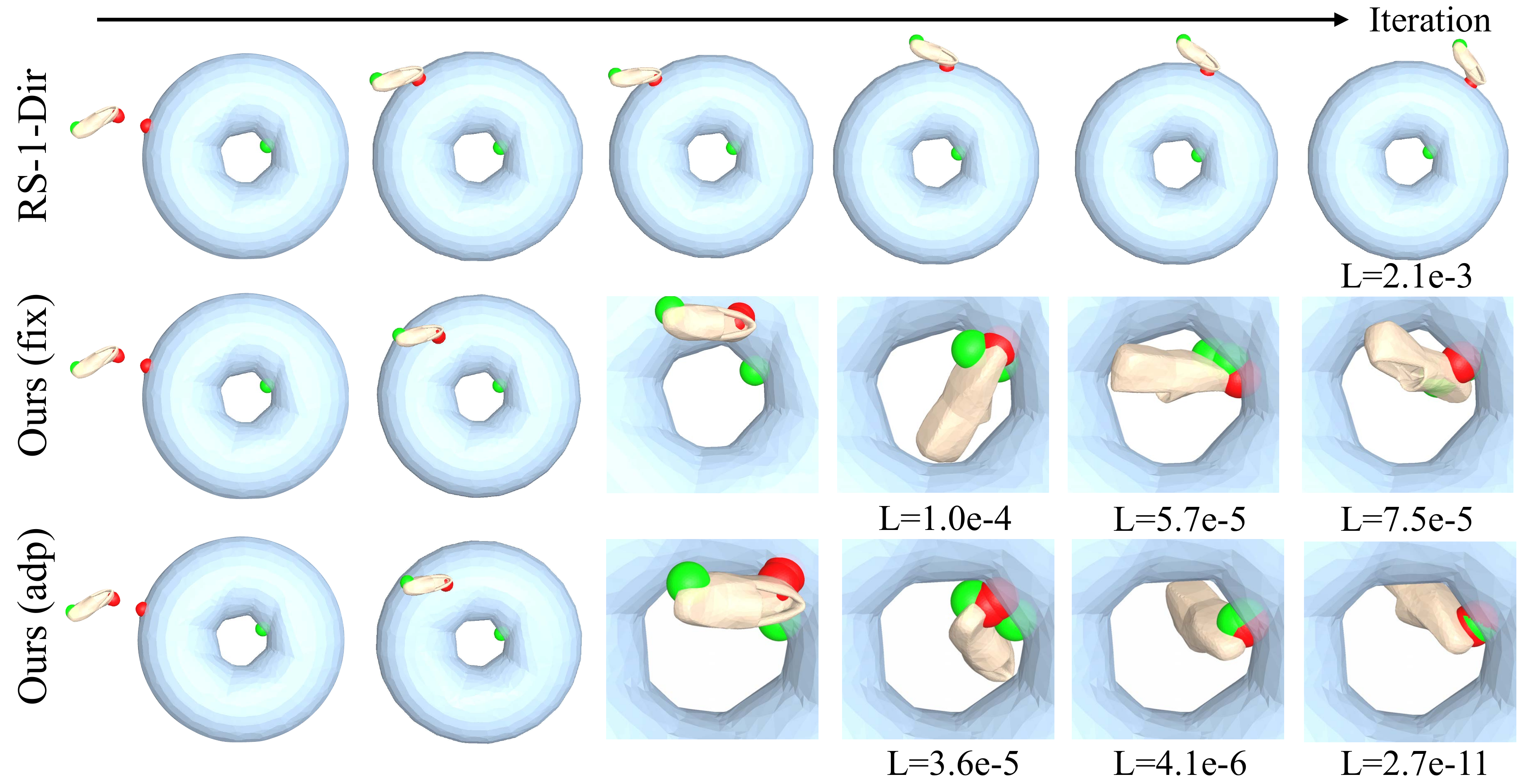}
    \caption{\textbf{Qualitative comparison on concave objects.} \textit{RS-1-Dir} fails dramatically when the target points lie on concave regions, regardless of the sampling strategy. Our proposed \textit{RS-1-Dist} with adaptive sampling can converges well.}
    \label{fig: concave exp}
    \vspace{-2mm}
\end{figure*}

\subsection{Comparison with Baselines}

We compare against four baselines:  
\paragraph{Analytical} computes vertex–face derivatives at witness points, representing a special case of brute-force vertex–face and edge–edge check. Since brute-force cannot find witness points during penetration, we still use GJK as the forward function. In practice, witness points rarely lie on edge–edge pairs, possibly due to forward-backward mismatch and numerical issues, so only vertex–face derivatives are considered.
\paragraph{Finite Difference} perturbs each input dimension and computes central differences. 
\paragraph{RS-0} 0th-order random smoothing~\cite{montaut2022differentiable}, which can be viewed as an enhanced finite difference using Gaussian samples as perturbations. Due to its computational cost, we run RS-0 and Finite Difference with only 128 target points (1/8 of the other settings) to avoid excessive runtime.
\paragraph{RS-1-Dir} 1st-order direction-based random smoothing with neighbor-based sampling~\cite{montaut2022differentiable}. To make it parallel-friendly, neighboring vertices within the fifth level set are randomly subsampled to maintain a fixed number per batch.

As shown in Table~\ref{tab: main}, our method clearly outperforms all baselines. We achieve a median error below 0.1 mm across all tasks and outperform the baselines by more than $40\%$ in terms of mm-level accuracy. 

Among baselines, \textbf{Analytical} performs the worst and often gets stuck on vertices, as shown in Figure~\ref{fig: convex exp}. This occurs because witness points located on vertices have zero derivatives—small pose changes do not alter the witness. Even when the derivatives are nonzero, they depend on a single face, making the results highly sensitive to mesh quality.

\textbf{RS-0} outperforms \textbf{Finite Difference} due to more diverse perturbations providing better derivative estimates. However, they often fail to resolve initial penetrations, as shown in Figure~\ref{fig: convex exp}. Moreover, both methods are extremely time-consuming (Section~\ref{sec: time}) and thus limiting practical use. We also find that these two methods are very sensitive to the optimization step size
\ifWithAppendix
(see Appendix~\ref{apd: step size})
\fi
.

\textbf{RS-1-Dir} performs better than RS-0 and FD on convex objects but worse on concave ones, consistent with our analysis that direction-based randomized smoothing is inherently limited to convex geometries.  Moreover, we identify another major limitation of direction-based smoothing—its strong sensitivity to the contact margin 
\ifWithAppendix
(see Appendix~\ref{apd: contact margin})
\fi
. The results in Table~\ref{tab: main} are reported under its best-performing configuration, i.e., with a contact margin of $\beta = 10^{-3}$.

\begin{table}[]
    \centering
    \begin{tabular}{l|l|cccc}
    \hline
    \multicolumn{2}{l|}{
    \multirow{2}{*}{Acc@1e-6 (\%) $\uparrow$}} & \multicolumn{2}{c}{DexGraspNet}  & \multicolumn{2}{c}{Objaverse} \\
    \multicolumn{2}{l|}{} & Convex & Concave & Convex & Concave\\
    \hline 
    \multirow{3}{*}{Dir} 
    & Neighbor & 53.7 & 37.9 & 57.9 & 26.1 \\
    & Fixed & 63.2 & 52.5 & 64.0 & 37.8 \\
    & Adaptive & 62.2 & 55.8 & 61.4 & 41.9 \\
    \hline
    \multirow{3}{*}{Dist} 
    & Neighbor & 85.6 & 70.2 & 88.5 & 53.1 \\
    & Fixed & 89.9 & 75.1 & 91.0 & 55.8 \\
    & Adaptive & \textbf{91.2} & \textbf{80.6} & \textbf{89.5} & \textbf{62.4} \\
    \hline
    \hline
    \multirow{3}{*}{Dir} 
    & $T_1, T_2$ & 71.2 & 63.3 & 68.9 & 46.1 \\
    & $T_2$(w/o EG) & 29.8 & 26.6 & 33.4 & 21.3 \\
    & $T_2$(w/ EG) & 62.2 & 55.8 & 61.4 & 41.9 \\
    \hline
    \multirow{3}{*}{Dist} 
    & $T_1, T_2$ & 90.9 & \textbf{80.6} & \textbf{89.5} & \textbf{62.6} \\
    & $T_2$(w/o EG) & 70.8 & 64.4 & 73.3 & 50.4 \\
    & $T_2$(w/ EG) & \textbf{91.2} & \textbf{80.6} & \textbf{89.5} & 62.4 \\
    \hline
    \end{tabular}
    \caption{\textbf{Ablation study.} (1) \textit{Smoothing:} our distance-based formulation consistently outperforms the direction-based baseline~\cite{montaut2022differentiable}. (2) \textit{Sampling:} adaptive $\alpha$ improves accuracy, especially for concave objects. (3) \textit{Equivalent gradient (EG):} when only $T_2$ is optimized, EG recovers performance comparable to jointly optimizing $T_1$ and $T_2$.}
    \label{tab:ablation}
\end{table}

\subsection{Ablation Study}

In this section, we evaluate three design choices: 
\begin{itemize}
    \item \textbf{Smoothing strategy}: direction-based smoothing~\cite{montaut2022differentiable} vs. distance-based smoothing (ours)
    \item \textbf{Sampling}: neighbor-based~\cite{montaut2022differentiable}, fixed $\alpha$, adaptive $\alpha$
    \item \textbf{Equivalent gradient (EG)}: optimizing both $T_1, T_2$ vs. only $T_2$ with/without EG
\end{itemize}
The default setting is Distance + Adaptive + $T_2$(w/ EG). Since direction-based smoothing is highly sensitive to the contact margin, we report all results in this section as the average performance under two contact margin settings, $\beta = 0$ and $\beta = 10^{-3}$
\ifWithAppendix
(see Appendix~\ref{apd: contact margin} for details)
\fi
.

Table~\ref{tab:ablation} shows three key findings.  
First, our distance-based smoothing consistently outperforms the direction-based baseline across all settings, even when the baseline is enhanced with our adaptive sampling and equivalent gradient transport. As shown in Figure~\ref{fig: concave exp}, direction-based smoothing often fails dramatically on challenging concave objects. A possible reason the baseline does not completely fail on concave objects is that some randomly sampled target points happen to lie on locally convex regions.

Second, neighbor-based sampling performs poorly, likely due to its sensitivity to mesh quality, making it unsuitable for in-the-wild objects. Both fixed and adaptive $\alpha$ perform well on convex shapes, while adaptive sampling is especially beneficial for concave ones.

Finally, equivalent gradient transport is essential when only $T_2$ is optimized. Without EG, performance degrades due to the mismatch between the current and target contact point on object 1, since moving an object itself typically requires smaller pose changes than moving another object. EG compensates for this effect and restores performance to the level of jointly optimizing $T_1$ and $T_2$.

\begin{table}[]
    \vspace{2mm}
    \centering
    \begin{tabular}{l|cccc}
    \hline
    \multirow{2}{*}{Time (\textmu s)}   & \multicolumn{2}{c}{DexGraspNet}  & \multicolumn{2}{c}{Objaverse} \\
     & Convex & Concave & Convex & Concave\\
    \hline 
       Forward & 8.3 & 23.4 & 8.7 & 29.2 \\
     Backward (RS-0 family) & 58.6 & 488.3 & 59.3 & 1269.5\\
      Backward (RS-1 family) & 17.7 & 19.6 & 18.9 & 21.4 \\
    \hline 
    \end{tabular}
    \caption{\textbf{Runtime analysis.} RS-1 family (Ours) is much more efficient than RS-0, especially for concave objects.}
    \label{tab: time all}
    \vspace{-1mm}
\end{table}

\begin{table}[]
    \centering
    \begin{tabular}{c|ccc|cc|c}
    \hline 
    & \multicolumn{3}{c|}{Sampling} & \multicolumn{2}{c|}{Derivative} & \multirow{2}{*}{EG} \\
        & Neighbor & Fixed & Adaptive & Dir & Dist &  \\
    \hline 
        Time (\textmu s) & 13.5 & 9.8 & 10.4 & 5.7 & 4.5 & 1.1 \\
    \hline 
    \end{tabular}
    \caption{\textbf{Runtime breakdown of backward components.} Neighbor-based sampling runs in C++ with OpenMP on CPU; all other components run in PyTorch on GPU.}
    \label{tab: time detail}
    \vspace{-1mm}
\end{table}

\subsection{Runtime Analysis}
\label{sec: time}
We evaluate five methods grouped into two families, reporting representative runtimes in Table~\ref{tab: time all} since intra-family variances are negligible. The \textbf{RS-0 family} (Finite Difference, RS-0) is computationally expensive as each backward pass requires numerous forward collision detections. This becomes prohibitive for complex concave objects, such as those in the Objaverse dataset. Conversely, the \textbf{RS-1 family} (Analytical, RS-1-Dir, and \textbf{Ours}) maintains consistent efficiency regardless of mesh complexity. All benchmarks were performed using an Nvidia 3090 GPU and Intel Xeon Platinum 8255C CPU. Table~\ref{tab: time detail} shows a detailed runtime breakdown of each backward component. Sampling is time-consuming as it computes distances from the witness point to the entire pre-sampled candidate set, with complexity linear in the number of candidates. However, this operation is easily parallelizable on modern GPUs, mitigating the overhead.

\subsection{Preliminary Results on Dexterous Grasp Synthesis}
\label{sec: applic}
To further demonstrate the applicability of our method to downstream robotic tasks, we conduct a preliminary study on dexterous grasp synthesis. A grasp generated by Dexonomy~\cite{chen2025dexonomy} is used as initialization, and target contact point pairs are manually annotated on the object and fingertips. The optimizable variables include the hand 6D root pose and joint angles, with the loss term defined in Eq.~\ref{eq: loss}. Without any modification, our method allows gradients to propagate from the loss through the fingertip poses—computed via differentiable forward kinematics—back to the root pose and joint angles in an end-to-end manner.

As shown in Figure~\ref{fig: application}, our method successfully refines the grasp to make all annotated target points in contact. The current formulation relies on manually annotated target points, which limits scalability. However, since our method is not restricted to a specific loss (particularly in the fixed sampling setting), it could be combined with differentiable grasp quality metrics~\cite{chen2024task, chen2025bodex} to eliminate the need for manual annotation. We leave this integration for future work.  

\begin{figure}
    \vspace{2mm}
    \centering
    \includegraphics[width=\linewidth]{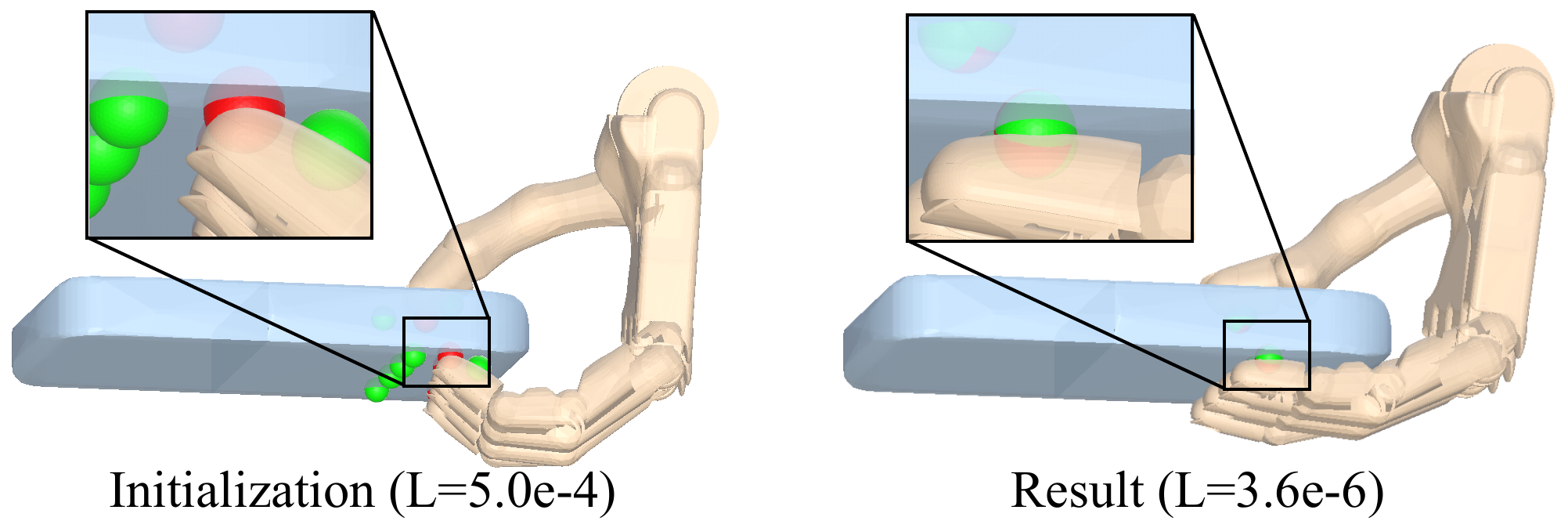}
    \caption{\textbf{Preliminary grasp refinement with our method.} Starting from a grasp generated by Dexonomy~\cite{chen2025dexonomy} and manually annotated target contact points, our method backpropagates gradients from Eq.~\ref{eq: loss} to the hand root pose and joint angles, successfully bringing the targets into contact.}
    \label{fig: application}
\end{figure}

\begin{figure}
    \vspace{2mm}
    \centering
    \includegraphics[width=\linewidth]{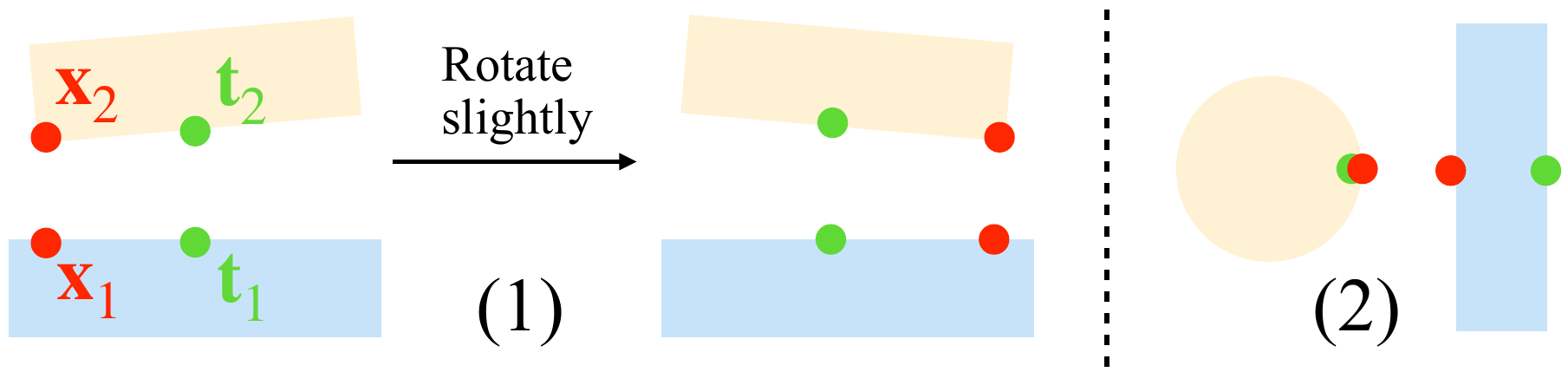}
    \caption{\textbf{Failure cases.} (1) Discontinuity of point contact model for face-face contact. (2) A local minimum of the loss term that traps the optimization.}
    \label{fig: failure case}
\end{figure}

\section{Limitations and Future Works}
Our method still faces several limitations. First, when both target points lie on planar faces, convergence to low loss is difficult because face–face contact is inherently discontinuous in the point–contact model (Figure~\ref{fig: failure case}, Left). Hydroelastic contact models~\cite{elandt2019pressure, masterjohn2022velocity} used in Drake~\cite{drake} provide a smoother alternative, though they currently do not extend to non-contact scenarios and incur higher computational cost. Second, optimization may become trapped in local minima when the target point lies on the opposite side of the object (Figure~\ref{fig: failure case}, Right). Finally, while our experiments demonstrate robustness on a toy problem, further validation is needed to assess applicability in gradient-based pipelines for tasks like dexterous grasping and manipulation.

\section{Conclusions}
In this work, we presented a general approach for differentiating witness points from collision detection, extending beyond convex shapes to arbitrary meshes. Experiments on standard object datasets showed robustness across object scale, configuration, and mesh complexity. We further provided preliminary results on applying our method to a downstream task, dexterous grasp refinement. Future work will focus on integrating our method into planning and control frameworks to support downstream contact-rich tasks in robotic grasping and manipulation.


\bibliographystyle{IEEEtran}
\bibliography{IEEEabrv}

\ifWithAppendix
\appendix
\subsection{Introduction}

In the main paper, we reported near-optimal settings for the baselines—contact margin $\beta = 10^{-3}$ and step size $s_r = 10$ for rotation and $s_t = 0.01$ for translation—to provide a reasonably fair comparison. However, baselines are generally much more sensitive to these two hyperparameters than our method, so the actual performance gap is often larger than what is reported. In this appendix, we present additional experiments that systematically vary contact margin and step size, highlighting the robustness of our method across a broader range of conditions.

\begin{figure}[h]
\centering
\includegraphics[width=\linewidth]{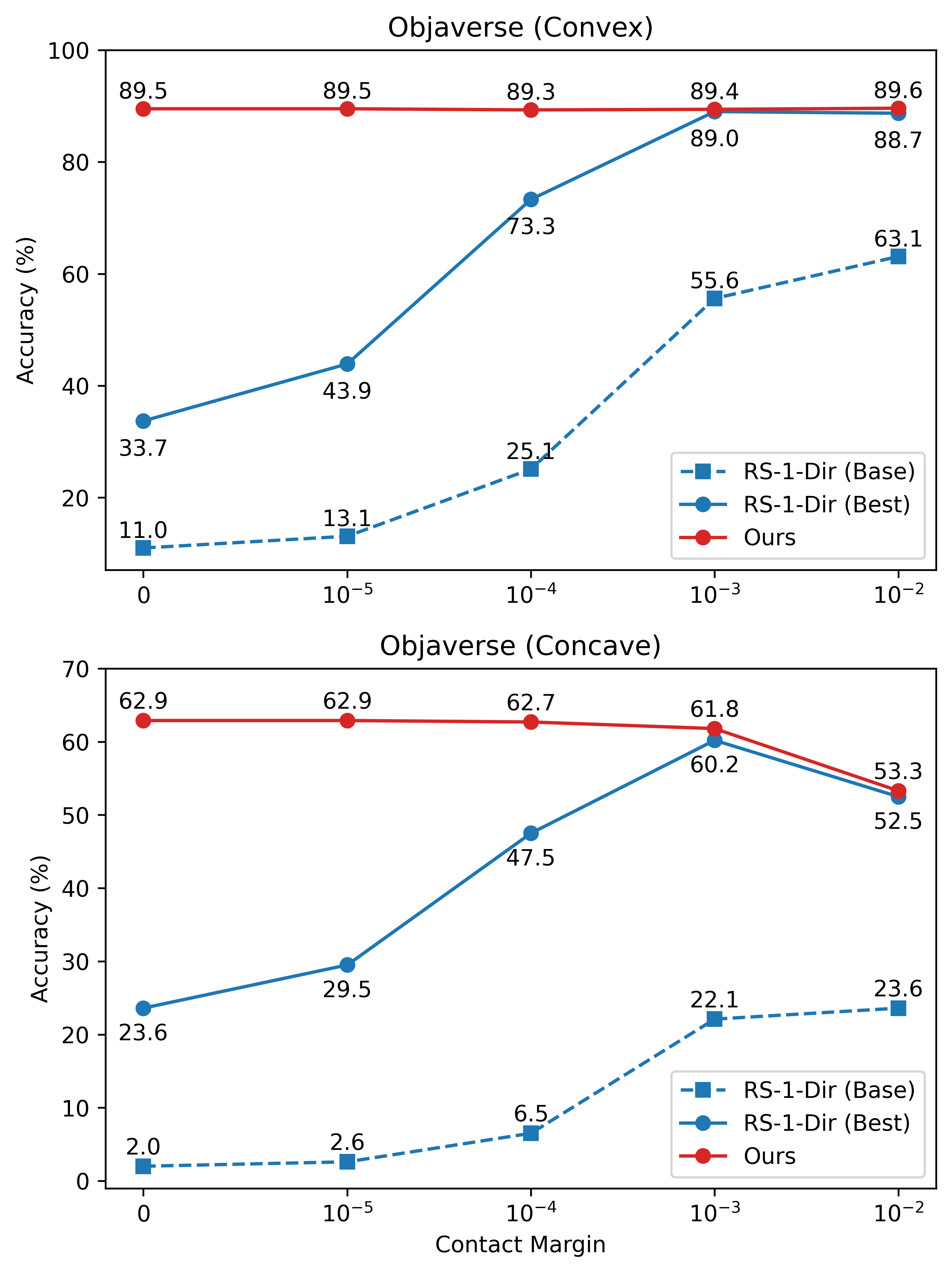}
\caption{\textbf{Robustness to contact margin.} Our method is robust across different contact margins, whereas RS-1-Dir requires a sufficiently large margin for reasonable performance, even with adaptive sampling and equivalent gradient.}
\label{fig: contact margin}
\end{figure}

\subsection{Robustness to Contact Margin}

\label{apd: contact margin}

As discussed in Section~\ref{sec: compare}, direction-based smoothing can suffer from numerical instability when two witness points are very close—a situation common in robotics applications. Here, we provide experimental evidence supporting this claim and further demonstrate the robustness of our approach.

We first introduce a modified loss function based on Eq.~\ref{eq: loss} that incorporates the contact margin:
\begin{equation}
L(T_1, T_2) = \|\mathbf{x}_1 - \mathbf{x}_2 + \beta \mathbf{n}\|^2
+ \|\mathbf{x}_1 - \mathbf{t}_1\|^2
+ \|\mathbf{x}_2 - \mathbf{t}_2\|^2,
\label{eq: new loss}
\end{equation}
where the contact normal $\mathbf{n}$ is defined as
\begin{equation}
\mathbf{n} = \frac{\mathbf{x}_2 - \mathbf{x}_1}{\|\mathbf{x}_2 - \mathbf{x}_1\|},
\end{equation}
and $\beta$ is a positive hyperparameter representing the contact margin. When $\beta=0$, this formulation reduces to Eq.~\ref{eq: loss}.

\textbf{Experiment setup}.
To evaluate the effect of different contact margins, we vary $\beta$ over ${0, 10^{-5}, 10^{-4}, 10^{-3}, 10^{-2}}$ (in meters). On the Objaverse dataset, we compare our method with two variants of RS-1-Dir~\cite{montaut2022differentiable}::
\begin{itemize}
\item \textbf{Base}: Direction + Neighbor. The original implementation from~\cite{montaut2022differentiable}.
\item \textbf{Best}: Direction + Adaptive + EG. Enhanced with our adaptive sampling and equivalent gradient transport.
\end{itemize}

\textbf{Results and analysis}.
Figure~\ref{fig: contact margin} shows that our method maintains stable high performance across all contact margins for both convex and concave objects. RS-1-Dir, however, performs extremely poorly when $\beta=0$. As $\beta$ increases, direction vanishing is alleviated and performance improves, but it remains consistently below our method. Other baselines (i.e., FD and RS0) are less sensitive to contact margin and are not included here.

An interesting observation occurs at $\beta = 10^{-2}$, where both our method and RS-1-Dir (Best) show performance drops on concave objects. This is due to increased discontinuity of witness points in concave regions: a large $\beta$ can make some target point pairs ill-posed. As illustrated in Figure~\ref{fig: contact margin concave}, if a concave region of object $O_1$ has radius $r$, then when $\beta > r$, witness points are forced to the sharp corners rather than inside the concave region, making some correspondences impossible and reducing performance.

\begin{figure}[t]
\centering
\includegraphics[width=\linewidth]{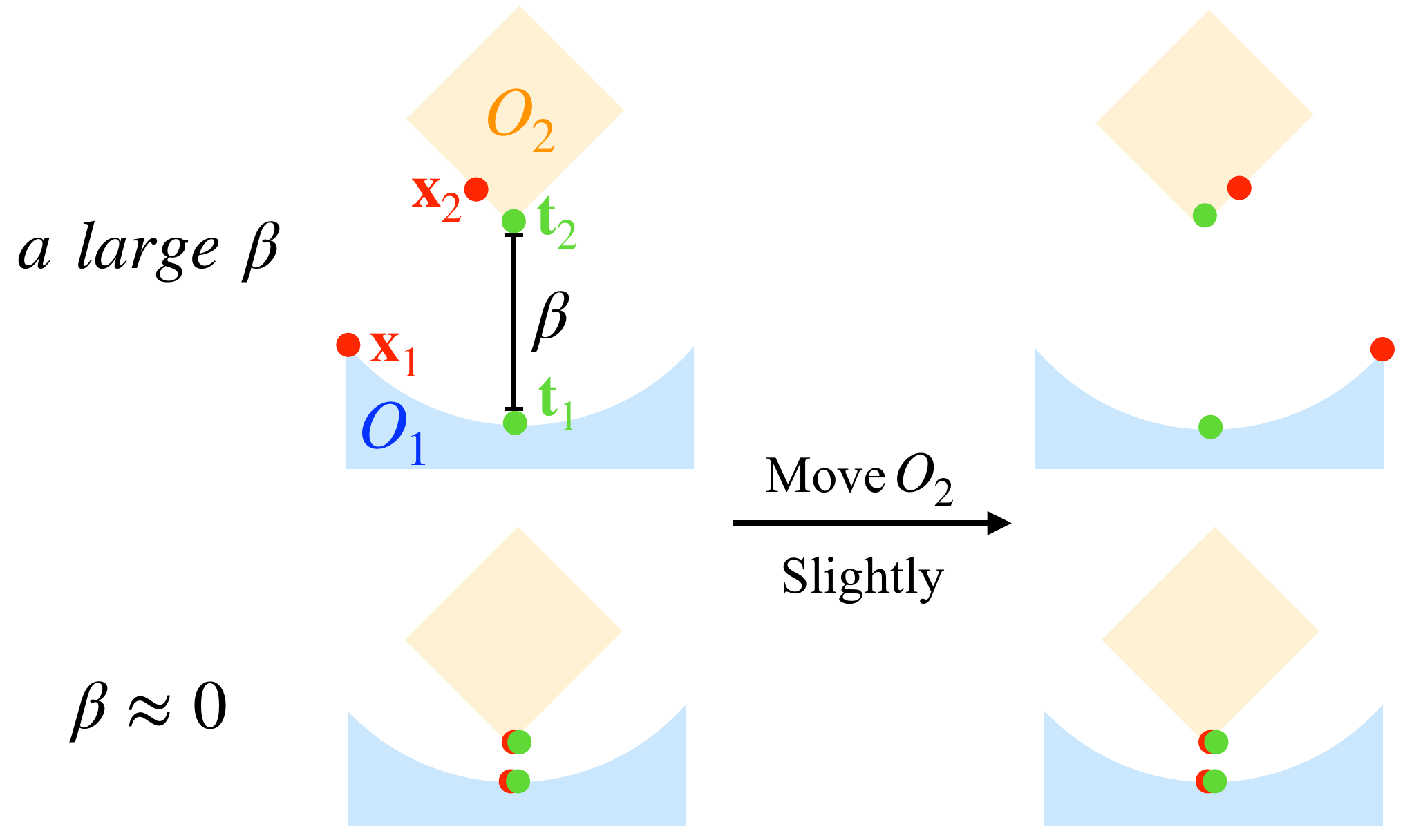}
\caption{\textbf{Performance drop for large contact margins on concave objects.} A large contact margin increases discontinuity in witness points, making some target points unreachable.}
\label{fig: contact margin concave}
\end{figure}

\subsection{Robustness to Step Size}
\label{apd: step size}

We also investigate the effect of optimization step size on performance.

\textbf{Experiment setup}. We vary the initial rotation step size $s_r$ across ${1, 2, 5, 10, 20, 50, 100}$, with translation step size $s_t = s_r / 100$. In addition to the two variants of RS-1-Dir, we include FD and RS-0, each with \textbf{Base} (original) and \textbf{Best} (enhanced with equivalent gradient). The contact margin is fixed at $10^{-3}$ to maximize RS-1-Dir performance.

\textbf{Result and analysis} Figure~\ref{fig: step size} shows that our method is the most robust to varying step sizes. The equivalent gradient consistently improves performance in almost all settings, with larger benefits for first-order methods than zero-order methods.

Baselines are more sensitive to step size, likely because it interacts with other hyperparameters that must be tuned jointly. For example, for the least robust baseline RS-0, the perturbation range is closely tied to the step size and must be adjusted accordingly, which complicates turning. In contrast, our method maintains consistently high performance across a wide range of step sizes without modifying any additional parameters, further demonstrating its robustness.

\begin{figure}[H]
    \centering
    \includegraphics[width=\linewidth]{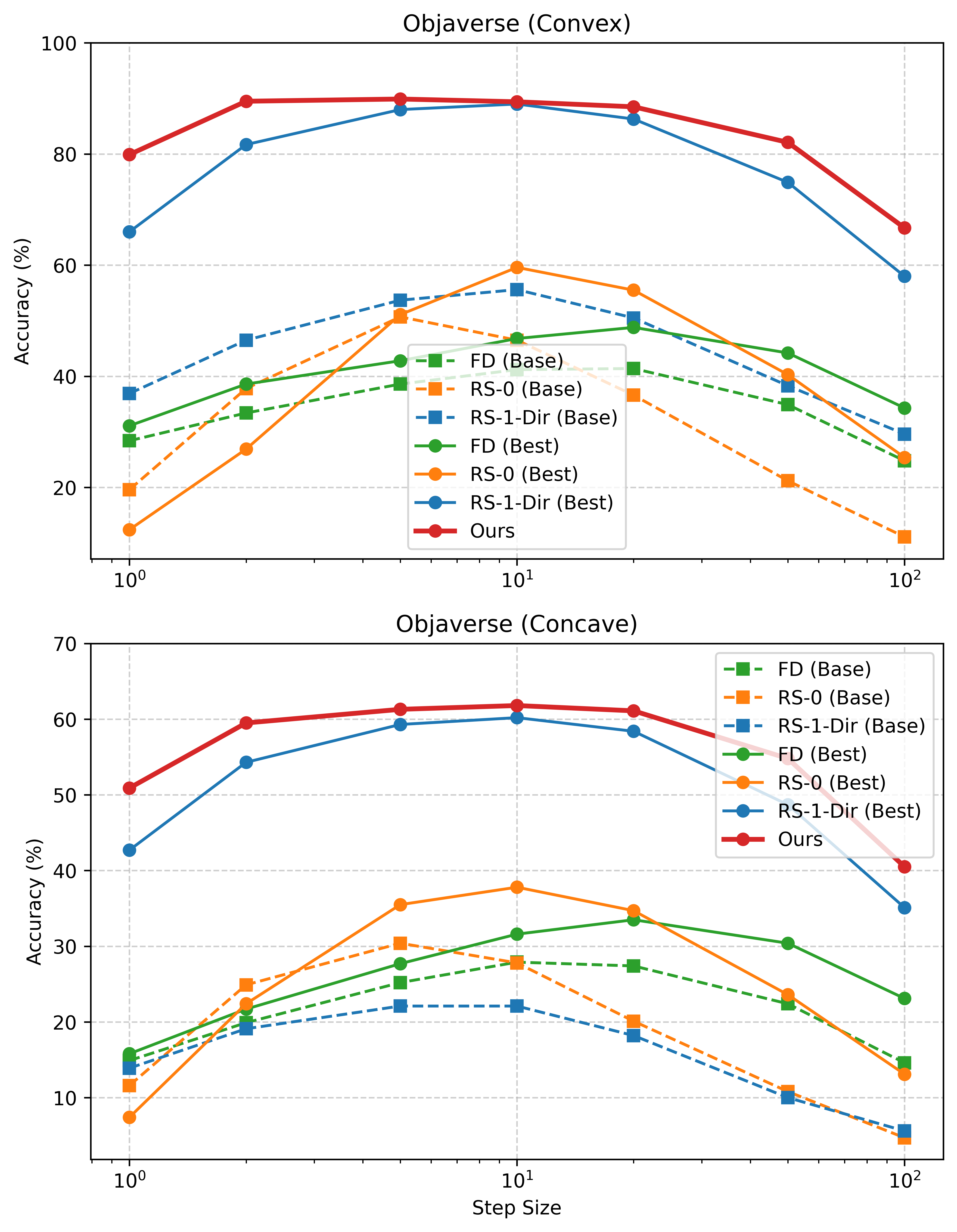}
    \caption{\textbf{Robustness to step size.} Our method maintains high performance across a wide range of step sizes without extra parameter tuning, indicating strong robustness.}
    \label{fig: step size}
\end{figure}
\fi

\end{document}